\documentclass[journal]{IEEEtran}

\usepackage{amsmath}
\usepackage{amssymb}
\usepackage{amsthm}
\usepackage{bm}
\usepackage{cite}
\usepackage{flushend}
\usepackage{graphicx}
\usepackage{hyperref}
\usepackage{mathrsfs}
\usepackage{mathtools}
\usepackage{titlesec}
\usepackage{xcolor}

\usepackage{macros}

\usepackage[capitalise, nameinlink]{cleveref}
\hypersetup{bookmarksopen=true}

\crefformat{equation}{\textup{#2(#1)#3}}
\crefrangeformat{equation}{\textup{#3(#1)#4--#5(#2)#6}}
\crefmultiformat{equation}{\textup{#2(#1)#3}}{ and \textup{#2(#1)#3}}
{, \textup{#2(#1)#3}}{, and \textup{#2(#1)#3}}
\crefrangemultiformat{equation}{\textup{#3(#1)#4--#5(#2)#6}}%
{ and \textup{#3(#1)#4--#5(#2)#6}}{, \textup{#3(#1)#4--#5(#2)#6}}{, and \textup{#3(#1)#4--#5(#2)#6}}

\Crefformat{equation}{#2Equation~\textup{(#1)}#3}
\Crefrangeformat{equation}{Equations~\textup{#3(#1)#4--#5(#2)#6}}
\Crefmultiformat{equation}{Equations~\textup{#2(#1)#3}}{ and \textup{#2(#1)#3}}
{, \textup{#2(#1)#3}}{, and \textup{#2(#1)#3}}
\Crefrangemultiformat{equation}{Equations~\textup{#3(#1)#4--#5(#2)#6}}%
{ and \textup{#3(#1)#4--#5(#2)#6}}{, \textup{#3(#1)#4--#5(#2)#6}}{, and \textup{#3(#1)#4--#5(#2)#6}}

\newtheorem{theorem}{Theorem}

\newtheorem{proposition}[theorem]{Proposition}

\theoremstyle{definition}
\newtheorem{definition}[theorem]{Definition}
\newtheorem{remark}[theorem]{Remark}
\newtheorem{example}[theorem]{Example}

\titlespacing\section{0pt}{6pt plus 4pt minus 2pt}{2pt plus 2pt minus 2pt}
\titlespacing\subsection{0pt}{6pt plus 4pt minus 2pt}{2pt plus 2pt minus 2pt}
\titlespacing\subsubsection{0pt}{6pt plus 4pt minus 2pt}{2pt plus 2pt minus 2pt}

\makeatletter
\def\thm@space@setup{\thm@preskip=1pt
\thm@postskip=0pt}
\makeatother

\begin{document}

\title{The Role of Neural Network Activation Functions}

\author{Rahul~Parhi
        and~Robert~D.~Nowak,~\IEEEmembership{Fellow,~IEEE}%
        \thanks{This work is partially supported by AFOSR/AFRL grant FA9550-18-1-0166 and the NSF Research Traineeship (NRT) grant 1545481.}
        \thanks{The authors are with the Department of Electrical and Computer Engineering, University of Wisconsin--Madison,  Madison, WI 53706 (e-mail: rahul@ece.wisc.edu; rdnowak@wisc.edu).}}

\maketitle

\begin{abstract}
  A wide variety of activation functions have been proposed for neural networks. The Rectified Linear Unit (ReLU) is especially popular today.  There are many practical reasons that motivate the use of the ReLU. This paper provides new theoretical characterizations that support the use of the ReLU, its variants such as the leaky ReLU, as well as other activation functions in the case of univariate, single-hidden layer feedforward neural networks. Our results also explain the importance of commonly used strategies in the design and training of neural networks such as ``weight decay" and ``path-norm" regularization, and provide a new justification for the use of ``skip connections" in network architectures. These new insights are obtained through the lens of spline theory. In particular, we show how neural network training problems are related to infinite-dimensional optimizations posed over Banach spaces of functions whose solutions are well-known to be fractional and polynomial splines, where the particular Banach space (which controls the order of the spline) depends on the choice of activation function.

\end{abstract}

\begin{IEEEkeywords}
neural networks, regularization, activation functions, inverse problems
\end{IEEEkeywords}

\section{Introduction}
\IEEEPARstart{V}{ariants} of the well-known universal approximation theorem for neural networks state that \emph{any} continuous function can be approximated arbitrarily well by a single-hidden layer neural network, under mild conditions on the activation function~\cite{uat1, uat2, uat3, uat4, nn-approx-non-poly}. While such results show that most nonlinear activation functions suffice for universal approximation in the ultra-wide limit, it is clear that the sequence of approximating functions, as well as the nature of functions learned by fitting networks to data, depends strongly on the choice of activation. Recent work on the approximation theory of neural networks has characterized how approximation rates depend on the choice of activation function~\cite{approx-relu-relu-squared,dimension-independent-approx-bounds}. However, these results do not consider the \emph{practical problem} of understanding the properties of functions \emph{learned} by neural networks fit to data. In this paper, we consider this problem in the univariate, single-hidden layer case.

As neural networks provide a rich space of functions, learning with neural networks is inherently ill-posed. Thus, regularization plays an important role in neural network training. One of the most common regularizers is \emph{weight decay}~\cite{weight-decay}, which corresponds to the regularizer being the Euclidean norm of the network weights. Regularization is popular since neural networks trained with regularization often generalize well on new, unseen data~\cite{moody1992effective, generalization-deep-learning, regularization-matters}.

In this paper we show how regularization in the finite-dimensional space of \emph{neural network parameters} is actually the same as regularization in the infinite-dimensional \emph{space of functions}.  In particular, we show how training neural networks with appropriate regularization results in functions that are solutions to an infinite-dimensional \emph{variational problem} posed over functions, where the regularizer is then a seminorm defining a Banach space that depends on the choice of activation function. We consider univariate, single-hidden layer feedforward neural networks mapping $\R \to \R$ of the form
\begin{equation}
    x \mapsto \sum_{k=1}^K v_k \, \rho(w_k x - b_k) + c(x),
    \label{eq:neural-network}
\end{equation}
where $\rho: \R \to \R$ is a fixed \emph{activation function}, $K$ is the \emph{width} of the network, for $k = 1, \ldots, K$, $v_k, w_k \in \R$, $w_k \neq 0$ are the \emph{weights} and $b_k \in \R$ are the first layer \emph{biases}, and $c(\dummy)$ is a ``generalized bias''\footnote{We will later see that $c(\dummy)$ corresponds to a ``simple'' function, e.g., a low degree polynomial, which depends on the activation function.} term in the last layer.

Our results rely on the key observation that in the univariate case, single-hidden layer neural networks are essentially \emph{spline functions}. Indeed, a spline function admits a representation
\begin{equation}
    x \mapsto \sum_{k=1}^K v_k \, \rho(x - b_k) + c(x).
    \label{eq:spline}
\end{equation}
The key difference between \cref{eq:neural-network} and \cref{eq:spline} is that the atoms of the neural network are \emph{translates} and \emph{dilates} of the activation function, while the atoms of the spline are \emph{only translates} of the ``activation function''. To this end, we use tools from the recently developed variational framework of $\Ell$-splines~\cite{L-splines},  to show that single-hidden layer neural networks trained with appropriate regularization are solutions to certain variational inverse problems. The dilations by input layer weights play a key role in the design of the neural network regularizers.

\subsection{Contributions}
In this paper we introduce the notion of \emph{admissible} activation functions. Roughly speaking, these are activation functions that allow for a rigorous connection between conventional neural network training and variational problems over an associated Banach space. Common activation functions such as the popular Rectified Linear Unit (ReLU) and modifications such as the leaky ReLU~\cite{leaky-relu}, are admissible and thus each is associated with its particular Banach space.
    
We instantiate our main result and show that training single-hidden layer neural networks with particular \emph{power activation functions}, introduced in \cref{ex:power-activations}, which include the ReLU and the leaky ReLU, and appropriate weight regularization produce \emph{optimal fractional and polynomial splines} fits to the data. In other words, neural network training solves infinite-dimensional optimizations over the Banach spaces of functions of higher order bounded variation.  Crucially, the regularizers are variants of the well-known path-norm~\cite{path-norm} and weight decay~\cite{weight-decay} regularizers that are ``matched'' to the activation function. We also show that admissible activation functions are necessarily these power activation functions.

Furthermore, for activation functions such as the ReLU and leaky ReLU, the generalized bias term exactly corresponds to the well-known notion of \emph{skip connections}~\cite{skip-connections} and thus our result also provides theoretical insight into the use of skip connections in neural network architectures. Finally, another interesting result of this paper is that it suffices to simply train a (sufficiently wide) neural network to solve certain variational inverse problems as opposed to more standard multiresolution or grid-based approaches~\cite{inv-prob1, inv-prob2}.

\subsection{Related work}
The choice of activation function plays an important role in the efficacy of neural networks. While the traditional sigmoid activation function was used for many years, the ReLU activation has become the preferred choice. Its initial motivation was to promote sparsity (in the sense of decreasing the number of active neurons)~\cite{relu-sparse}. It has also been empirically observed that the training of neural networks is much faster with ReLU activations~\cite{relu-fast}. Furthermore, variants of the ReLU, such as the leaky ReLU~\cite{leaky-relu}, have been proposed to avoid the problem of \emph{vanishing gradients} in neural network training. 

More recently, several recent works have made connections between splines and neural networks. In particular, the authors of~\cite{relu-linear-spline} show that the ``connect-the-dots'' linear spline is a solution to the problem of training a single-hidden layer ReLU network with weight decay subject to data fitting constraints. Another related, but different work, is concerned with the ``optimal shaping'' of activation functions in \emph{deep neural networks}~\cite{representer-deep,aziznejad2019deep} in which the authors consider \emph{learnable activation functions} and show that linear spline activation functions satisfy a minimal second-order total variation criterion. In our own work in~\cite{ridge-splines}, we relate neural network training to a variational problem over a Banach space in the multivariate case. We remark that in the univariate case explored in this paper, a much broader class of activation functions are \emph{admissible}. This is discussed further in \cref{rem:restrictive}.

\section{Preliminaries} \label{sec:prelim}
Let $\Sch(\R)$ be the Schwartz space of smooth and rapidly decaying test functions on $\R$ with continuous dual $\Sch'(\R)$, the space of tempered distributions on $\R$. We will be interested in the space $\M(\R)$ of finite Radon measures on $\R$. The space $\M(\R)$ can be viewed as a subspace of $\Sch'(\R)$ with the norm
\[
  \norm{u}_{\M(\R)}
  = \sup_{\varphi \in \Sch(\R),
  \norm{\varphi}_{L^\infty(\R)} = 1} \ang{u, \varphi},
\]
which is exactly the \emph{total variation norm in the sense of measures}. We are interested in $\M(\R)$ since it is a ``generalization'' of $L^1(\R)$. Indeed, we have $L^1(\R) \subset \M(\R)$ and for any $f \in L^1(\R)$ we have $\norm{f}_{L^1(\R)} = \norm{f}_{\M(\R)}$, but the translated Dirac impulses $\delta(\dummy - x_0)$, $x_0 \in \R$, are not in $L^1(\R)$ but are in $\M(\R)$ with $\norm{\delta(\dummy - x_0)}_{\M(\R)} = 1$.

We will now state the relevant background from the framework of $\Ell$-splines~\cite{L-splines}.

\begin{definition}[Definition~1 of~{\cite{L-splines}}] \label{defn:spline-admissible}
  A linear operator $\Ell: \Sch'(\mathbb{R}) \to
  \Sch'(\mathbb{R})$ is called \emph{spline-admissible} if
  \begin{enumerate} \tightlist
    \item it is translation-invariant, i.e., $\Ell\Tr_{x_0} = \Tr_{x_0}\Ell$, where $\Tr_{x_0}\curly{f}(x) = f(x - x_0)$ is the \emph{translation operator};
    \item there exists a function $\rho_{\Ell}: \mathbb{R} \to \mathbb{R}$ such
      that $\Ell \rho_{\Ell} = \delta$, i.e., $\rho_{\Ell}$ is a Green's
      function of $\Ell$;
\item the null space $\N_{\Ell} = \curly{q \st \Ell q = 0}$ has finite-dimension
      $N_0 \geq 0$.
  \end{enumerate}
\end{definition}

\begin{definition}[Definition~2 of~{\cite{L-splines}}] \label{defn:L-spline}
    A function $s: \R \to \R$ is said to be a \emph{nonuniform $\Ell$-spline} if
    \[
        \Ell\curly{s} = \sum_{k=1}^K v_k \, \delta(\dummy - b_k),
    \]
    where $\curly{v_k}_{k=1}^K$ is a sequence of weights and the locations of Dirac impulses are at the spline knots $\curly{b_k}_{k=1}^K$.
\end{definition}
\begin{remark}
    Notice that the spline representation in \cref{eq:spline} with $\rho$ being a Green's function of $\Ell$ is clearly a nonuniform $\Ell$-spline, so long as $c(\dummy) \in \N_{\Ell}$. The finite-dimensionality is required in \cref{defn:L-spline}, so that $c(\dummy)$ can be represented by a finite number of coefficients. We refer to the representation in \cref{eq:spline} as the \emph{canonical spline representation}.
\end{remark}

The fundamental result of~\cite{L-splines} is the following \emph{representer theorem} regarding the structure of the solutions to variational problems with generalized total variation regularization.

\begin{proposition}[Based on Theorems~1~and~2 of~{\cite{L-splines}}]
  \label{thm:rep-thm}
  Let $\Ell$ be a spline-admissible operator in the sense of \cref{defn:spline-admissible}. Then, the extreme points of the solutions of
  \begin{equation}
      \min_{f\in \M_{\Ell}(\R)} \: \norm{\Ell
      f}_{\M(\R)}
      \quad \subj \quad \ang{\nu_n, f} = y_n, \: n = 1, \ldots, N
    \label{eq:inverse-problem}
  \end{equation}
  are necessarily non-uniform $\Ell$-splines of the form in \cref{eq:spline} with the $K \leq N - N_0$ knots, where $\rho$ is a Green's function of $\Ell$ and $c(\dummy) \in \N_{\Ell}$, $\vec{\nu}: f \mapsto \paren{\ang{\nu_1, f}, \ldots, \ang{\nu_N, f}} \in \R^N$ is a weak$^*$-continuous \emph{measurement operator}, and $\M_{\Ell}(\R)$ is the \emph{native space} of $\Ell$ defined by $\M_{\Ell}(\R) \coloneqq \curly{f \in \Sch'(\R) \st \Ell f \in \M(\R)}$.
\end{proposition}
\begin{remark} \label{rem:regularized}
    For appropriate choices of loss function\footnote{A strictly convex, coercive, lower semi-continuous loss function suffices.}, the result of \cref{thm:rep-thm} also holds for \emph{regularized problems}:
    \begin{equation}
        \min_{f\in \M_{\Ell}(\R)} \: \sum_{n=1}^N \ell(y_n, \ang{\nu_n, f}) + \lambda\norm{\Ell f}_{\M(\R)}
        \label{eq:regularized}
    \end{equation}
    where $\ell(\dummy, \dummy)$ is the loss function and $\lambda > 0$ is an adjustable regularization parameter.
\end{remark}
\begin{remark} \label{rem:ideal-sampling}
    In machine learning, the measurement model is taken to be \emph{ideal sampling}, i.e., $\nu_n = \delta(\dummy - x_n)$ for some $x_n \in \R$. In other words, the machine learning problem considers fitting the data $\curly{(x_n, y_n)}_{n=1}^N \subset \R \times \R$. In the rest of this paper, we will only be interested in this setting. A sufficient condition for weak$^*$-continuity of $\delta(\dummy - x_n)$ is continuity of the Green's function of $\Ell$. For a detailed proof in the case that $\Ell = \D^2$, the second derivative operator, see~\cite[Theorem~1]{representer-deep}.
\end{remark}

\section{Neural Network Training and Regularization} \label{sec:main}
In this section we will state our main results.

\begin{definition} \label{defn:nn-admissible}
  A linear operator $\Ell: \Sch'(\mathbb{R}) \to
  \Sch'(\mathbb{R})$ is called \emph{neural network-admissible} if
  \begin{enumerate} \tightlist
    \item it is spline-admissible in the sense of \cref{defn:spline-admissible} with a continuous\footnote{See \cref{rem:ideal-sampling}.}  Green's function;
    \item there exists $g: \R \to \R$ such that $\Ell \Di_w = g(w) \Di_w \Ell$, where $\Di_w \curly{f}(x) \coloneqq f(wx)$ is the \emph{dilation operator}. \label{item:dilation-property}
  \end{enumerate}
\end{definition}
\begin{definition} \label{defn:admissible-activation}
  An activation function $\rho: \R \to \R$ is called \emph{admissible} if it is the continuous Green's function of some neural network-admissible operator.
\end{definition}

We see that single-hidden layer neural networks with admissible activation functions are in fact splines. Indeed, let $\rho$ be an admissible activation function for the neural network-admissible operator $\Ell$. Then, consider the neural network
\begin{equation}
    f_\vec{\theta}(x) = \sum_{k=1}^K v_k \, \rho(w_k x - b_k) + c(x),
    \label{eq:nn}
\end{equation}
where $\vec{\theta} = (v_1, \ldots, v_K, w_1, \ldots, w_K, b_1, \ldots, b_K, c)$ contains the neural network parameters and $c(\dummy) \in \N_{\Ell}$. Also, let $\Theta$ be the space of all neural network parameters $\vec{\theta}$. We see that
\begin{align*}
    \Ell\curly{f_\vec{\theta}}
    &= \sum_{k=1}^K v_k (\Ell \Di_{w_k})\curly{\rho(\dummy - b_k/w_k)} \\
    &= \sum_{k=1}^K v_k\, g(w_k) (\Di_{w_k} \Ell) \curly{\rho(\dummy - b_k/w_k)} \\
    &= \sum_{k=1}^K v_k\, g(w_k) \, \delta(w_k(\dummy)  - b_k) \\
    &= \sum_{k=1}^K v_k \frac{g(w_k)}{\abs{w_k}} \, \delta(\dummy - b_k/w_k) \numberthis \label{eq:nn-is-spline}
\end{align*}
where in the last line we used the fact that the Dirac impulse is homogeneous of degree $-1$. From \cref{defn:L-spline}, we see from \cref{eq:nn-is-spline} that $f_\vec{\theta}$ is an $\Ell$-spline with spline knots at $\curly{b_k / w_k}_{k=1}^K$. Thus, we see that although the neural network representation is not the canonical spline representation, neural networks, with admissible activation functions, are in fact splines. By \cref{thm:rep-thm}, this says that they are solutions to variational problems of the form in \cref{eq:inverse-problem}. We can now state our main result.
\begin{theorem} \label{thm:main}
    Let $\Ell$ be a neural network-admissible in the sense of \cref{defn:nn-admissible}, and let $\rho$ be a continuous Green's function of $\Ell$. Then, the solutions to
    \[
        \min_{\vec{\theta} \in \Theta} \: \sum_{k=1}^K \abs{v_k} \frac{\abs{g(w_k)}}{\abs{w_k}} \quad\subj\quad f_\vec{\theta}(x_n) = y_n, \: n = 1, \ldots, N
    \]
    with $K \geq N - N_0$ are solutions to the variational problem in \cref{eq:inverse-problem} under the ideal sampling setting.
\end{theorem}
\begin{proof}
    Consider a neural network as in \cref{eq:nn} and assume it is in \emph{reduced form}, i.e., the weight bias pairs $(w_k, b_k)$ are unique. The theorem follows by taking the $\norm{\dummy}_{\M(\R)}$ of \cref{eq:nn-is-spline}.
\end{proof}
\begin{remark}
    Just as in \cref{rem:regularized}, \cref{thm:main} also holds for regularized problems similar to \cref{eq:regularized}.
\end{remark}

\begin{example} \label{ex:power-activations}
    Consider the activation function defined by
    \begin{equation}
        \rho_{\alpha, \beta, \gamma}(x) \coloneqq \begin{cases}
            \alpha\, x^{\gamma-1}, & x < 0, \\
            \beta\, x^{\gamma-1}, & x \geq 0,
        \end{cases}
        \label{eq:power-activation}
    \end{equation}
    where $\alpha, \beta \in \R$ with $\alpha \neq \beta$ and $\gamma \geq 1$. We refer to this as an \emph{$(\alpha, \beta, \gamma)$-power activation function}, and refer to $\gamma$ as the \emph{order} of the activation function. This family of activation functions are admissible with corresponding operator being $\D^\gamma$, the $\gamma$th-order derivative operator, since, up to a constant factor, $\rho_{\alpha, \beta, \gamma}$ is a Green's function of $\D^\gamma$. When $\gamma$ is not an integer, $\D^\gamma$ is understood as the Fourier multiplier $\omega \mapsto (\imag \omega)^\gamma$. In this case, $g(w) = w^\gamma$. Hence, the corresponding regularizer is
    \begin{equation}
        \sum_{k=1}^K \abs{v_k} \frac{\abs{g(w_k)}}{\abs{w_k}} = \sum_{k=1}^K \abs{v_k}\abs{w_k}^{\gamma - 1},
        \label{eq:power-path-norm}
    \end{equation}
    which can be viewed as a generalized $\ell^1$-path-norm regularizer~\cite{path-norm} that is ``matched'' to the activation function.  This path-norm is also an upper bound on the Rademacher complexity of neural neural networks~\cite{ridge-splines}; thus networks with small path-norms have better generalization bounds.
\end{example}
\begin{theorem} \label{thm:power-necessary}
    An admissible activation function necessarily takes the form in \cref{eq:power-activation}.
\end{theorem}
\begin{proof}
    From \cref{item:dilation-property} in \cref{defn:nn-admissible}, we see that an admissible activation function $\rho: \R \to \R$ must satisfy
    \begin{equation}
        \rho(wx) = g(w) \rho(\sgn(w) x)
        \label{eq:must-satisfy}
    \end{equation}
    for some $g: \R \to \R$. Put $P(x) \coloneqq \ln \rho(e^x)$. For any $h \in \R$,
    \begin{align*}
        P(x + h) &= \ln \rho(e^{x + h})
        = \ln \rho(e^h e^x) \\
        &= \ln \curly{g(e^h) \rho(e^x)} = \ln g(e^h) + P(x),
    \end{align*}
    where in the second line we used the fact that $e^h > 0$ for all $h \in \R$. Next, fix $h \in \R \setminus \curly{0}$ and consider the finite difference
    \[
        \Delta_h\curly{P}(x) \coloneqq \frac{P(x + h) - P(x)}{h} = \frac{\ln g(e^h)}{h}.
    \]
    Since the finite difference is independent of $x$, we see that $P$ is piecewise linear. Consider an interval $I \subset \R$ in which $P(x) = a x + b$ for all $x \in I$ for some $a, b \in \R$. Then, for all $x \in I$ we have
    \[
        \rho(x) = e^{P(\ln x)} = e^{a \ln x + b} = e^b x^a.
    \]
    Finally, by \cref{defn:nn-admissible}, $\rho$ must be spline-admissible and must satisfy \cref{eq:must-satisfy}. It follows that $\rho$ must take the form in \cref{eq:power-activation}.
\end{proof}
\begin{remark}
    When $\gamma$ is not an integer, the functions learned by networks with $\rho_{\alpha, \beta, \gamma}$ activation functions trained on data and regularized according to \cref{eq:power-path-norm} are optimal \emph{$\gamma$th-order fractional splines}~\cite{fractional-splines} fit to the data. When $\gamma$ is an integer, the learned functions are optimal \emph{$\gamma$th-order polynomial splines}.
\end{remark}

\begin{example}
    When $(\alpha, \beta, \gamma) = (0, 1, 2)$, we have $\rho_{0, 1, 2} = \max\curly{0, \dummy}$ which is exactly the ReLU. The generalized bias term takes the form of a skip connection, i.e., $c(x) = ux + s$, where $u, s \in \R$ are trainable parameters. Additionally, the regularizer in \cref{eq:power-path-norm} is exactly the $\ell^1$-path-norm regularizer proposed in~\cite{path-norm}. This same result holds for modifications of the ReLU such as the leaky ReLU~\cite{leaky-relu}, which is a $(\alpha, 1, 2)$-power activation function. When trained on data, these networks learn functions that are optimal with respect to the Banach space of functions of second-order bounded variation which are \emph{optimal linear splines} fit to the data.
\end{example}

\begin{remark}
    The leaky ReLU was proposed in order to avoid the \emph{dying ReLU problem} in the training neural networks, where weights get stuck at $0$ due to the fact that the ReLU is $0$ for all inputs less than $0$. Since our result says that the underlying function spaces for the ReLU and leaky ReLU are the same, perhaps the leaky ReLU should be used over the ReLU.
\end{remark}

\begin{example} \label{ex:truncated-power}
    The truncated power functions given by $\rho_{0, 1, \gamma} \propto \max\curly{0, \dummy}^{\gamma-1} / (\gamma - 1)!$, where $\gamma$ is a positive integer, are admissible. The generalized bias term takes the form of a polynomial of degree less than $\gamma$, with trainable coefficients, which can be viewed as a generalized skip connection.
\end{example}


\begin{remark} \label{rem:restrictive}
    In our related work in~\cite{ridge-splines} we consider a similar problem to this paper, but in the multivariate case and relate training multivariate single-hidden layer networks to a variational problem over a Banach space. Our result there is more restrictive in that the only admissible activation functions are power activation functions where $\gamma$ is a postive even integer, and also does not make any connections to splines.
\end{remark}


Remarkably, as noticed in~\cite{ridge-splines}, is that the regularizer as in \cref{eq:power-path-norm} is related to the well-known weight-decay regularizer~\cite{weight-decay}.
\begin{proposition}[Special case of Proposition~2.13 of~{\cite{ridge-splines}}] \label{prop:equiv-opts}
    Consider training neural networks as in \cref{eq:nn} with an admissible activation function of order $\gamma$. Then, the following optimization problems are equivalent:
    \[
        \min_{\vec{\theta} \in \Theta} \: \sum_{k=1}^K \abs{v_k} \abs{w_k}^{\gamma-1} \quad\subj\quad f_\vec{\theta}(x_n) = y_n, \: n = 1, \ldots, N
    \]
    \[
        \min_{\vec{\theta} \in \Theta} \: \frac{1}{2}\sum_{k=1}^K \abs{v_k}^2 + \abs{w_k}^{2\gamma-2} \quad\subj\quad f_\vec{\theta}(x_n) = y_n, \: n = 1, \ldots, N
    \]
\end{proposition}
\begin{remark}
    These optimizations are also equivalent in the case of regularized problems similar to \cref{eq:regularized}.
\end{remark}

\begin{remark}
    When $\gamma = 2$, the second optimization in \cref{prop:equiv-opts} is exactly the well-known weight decay regularizer. Thus, ReLU networks and leaky ReLU networks are \emph{intrinsically tied} to the well-known weight decay regularizer.
\end{remark}
\section{Empirical Validation} \label{sec:empirical}
In this section we verify empirically that the claims made in \cref{sec:main} hold. We use \cref{prop:equiv-opts} and consider regularized neural network training problems of the form
\begin{equation}
    \min_{\vec{\theta} \in \Theta} \: \sum_{n=1}^N \abs{y_n - f_\vec{\theta}(x_n)}^2 + \frac{\lambda}{2} \sum_{k=1}^K \abs{v_k}^2 + \abs{w_k}^{2\gamma - 2}.
    \label{eq:training}
\end{equation}
To promote interpolation of the data we take $\lambda = 10^{-5}$. We specifically consider the ReLU activation which is a power activation function with $(\alpha, \beta, \gamma) = (0,1,2)$ and the cubic truncated power activation which is a power activation function with $(\alpha, \beta, \gamma) = (0, 1, 4)$. PyTorch was used to implement the networks and AdaGrad~\cite{adagrad} to train the networks.

\begin{figure}[htb]
\begin{minipage}[b]{0.48\linewidth}
  \centering
  \centerline{\includegraphics[width=4.0cm]{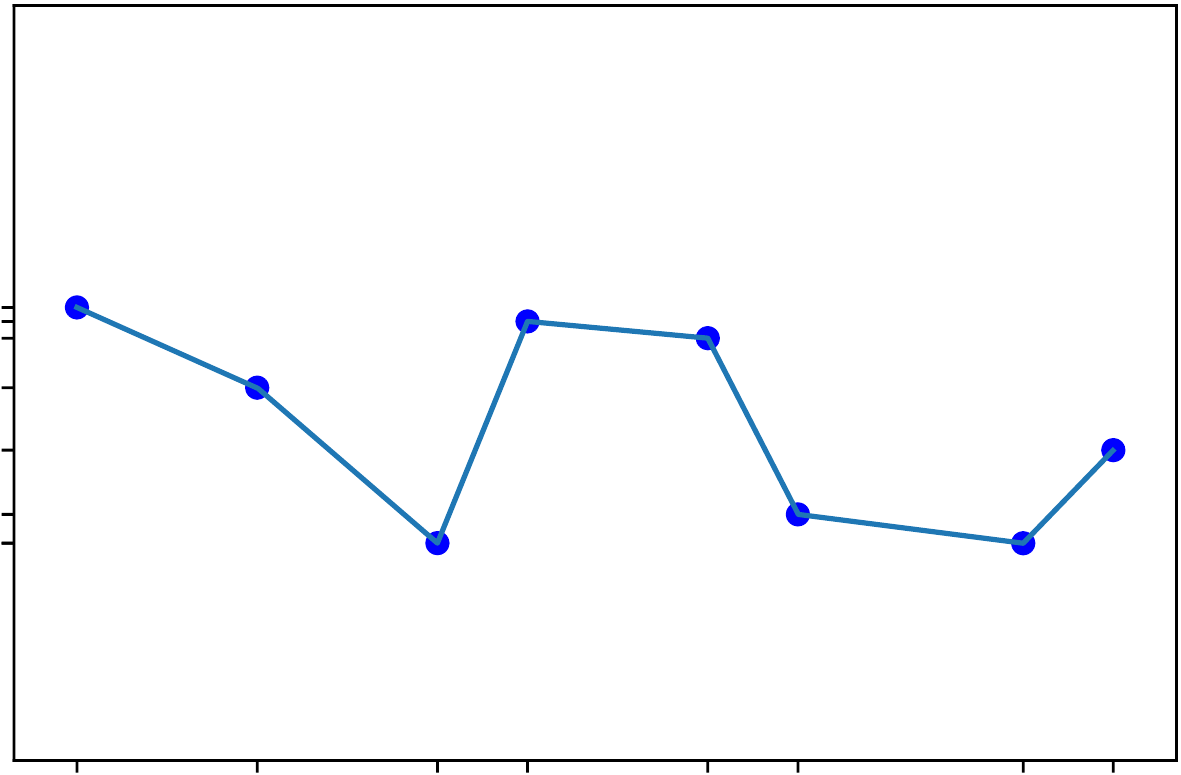}}
  (a) Standard linear spline \\
  $\norm{\D^2 f}_{\M(\R)} = 22.3$\smallskip
\end{minipage}
\hfill
\begin{minipage}[b]{0.48\linewidth}
  \centering
  \centerline{\includegraphics[width=4.0cm]{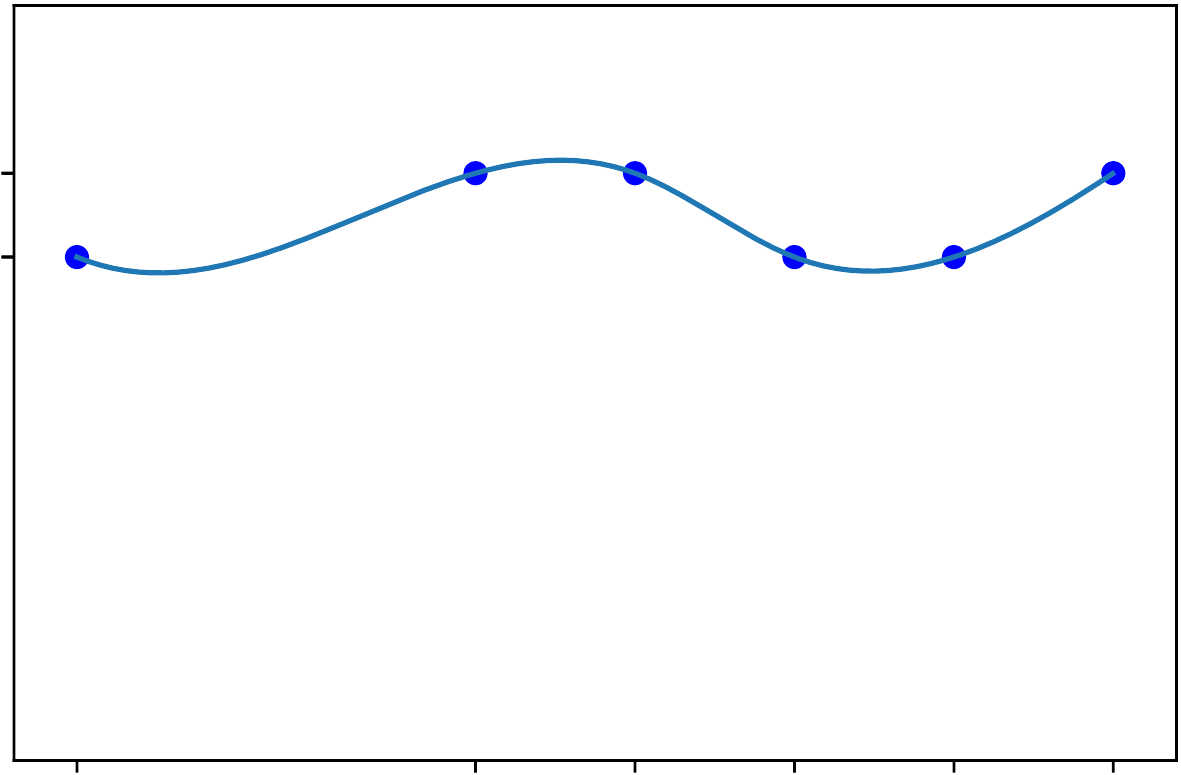}}
  (d) Standard cubic spline\\
  $\norm{\D^4 f}_{\M(\R)} = 5.8$\smallskip
\end{minipage}
\hfill
\begin{minipage}[b]{0.48\linewidth}
  \centering
  \centerline{\includegraphics[width=4.0cm]{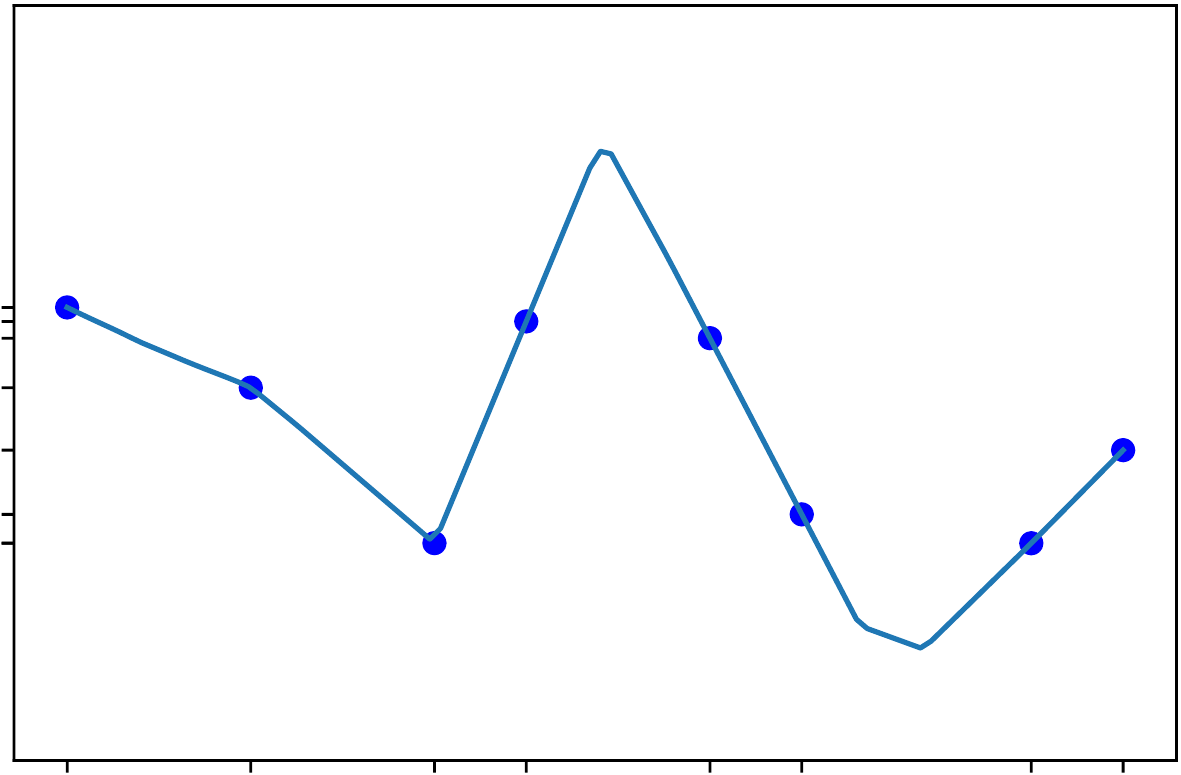}}
  (b) Neural network \\
  with regularization \\
  $\norm{\D^2 f_\vec{\theta}}_{\M(\R)} = 22.3$\smallskip
\end{minipage}
\hfill
\begin{minipage}[b]{0.48\linewidth}
  \centering
  \centerline{\includegraphics[width=4.0cm]{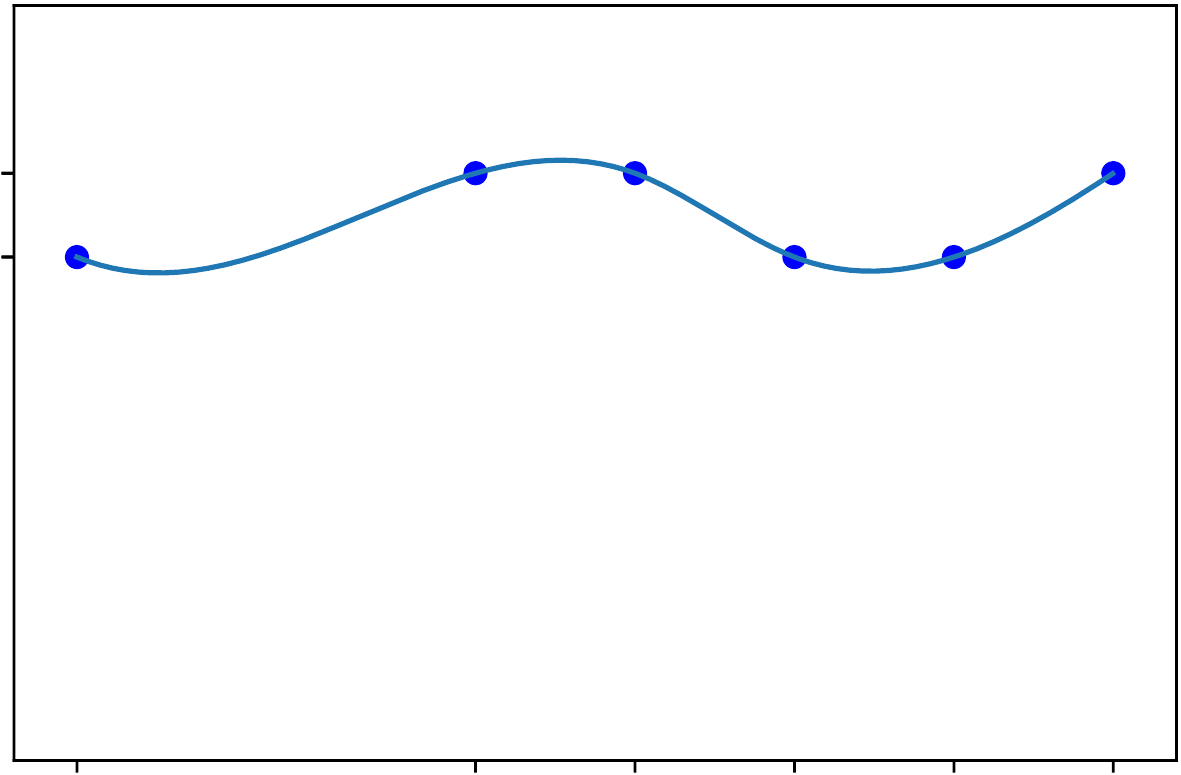}}
  (e) Neural network \\
  with regularization
  $\norm{\D^4 f_\vec{\theta}}_{\M(\R)} = 5.8$\smallskip
\end{minipage}
\hfill
\begin{minipage}[b]{0.48\linewidth}
  \centering
  \centerline{\includegraphics[width=4.0cm]{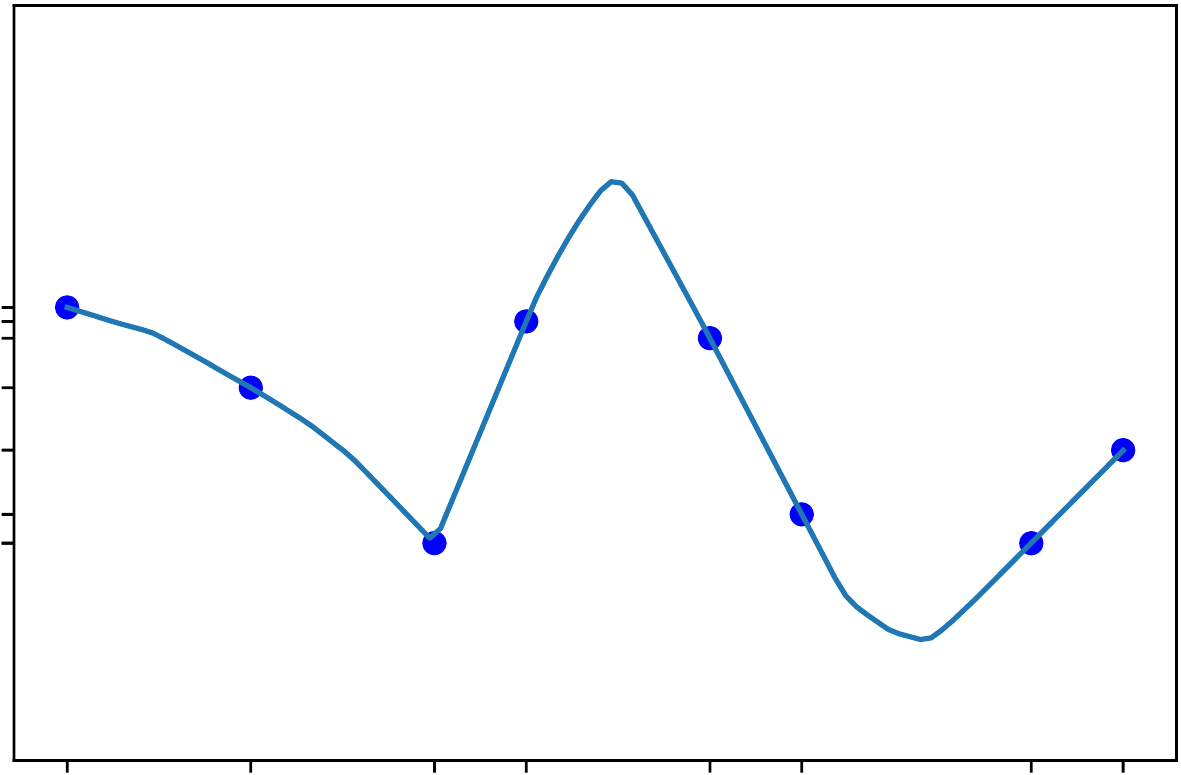}}
  (c) Neural network \\
  without regularization
  $\norm{\D^2 f_\vec{\theta}}_{\M(\R)} = 25.7$
\end{minipage}
\hfill
\begin{minipage}[b]{0.48\linewidth}
  \centering
  \centerline{\includegraphics[width=4.0cm]{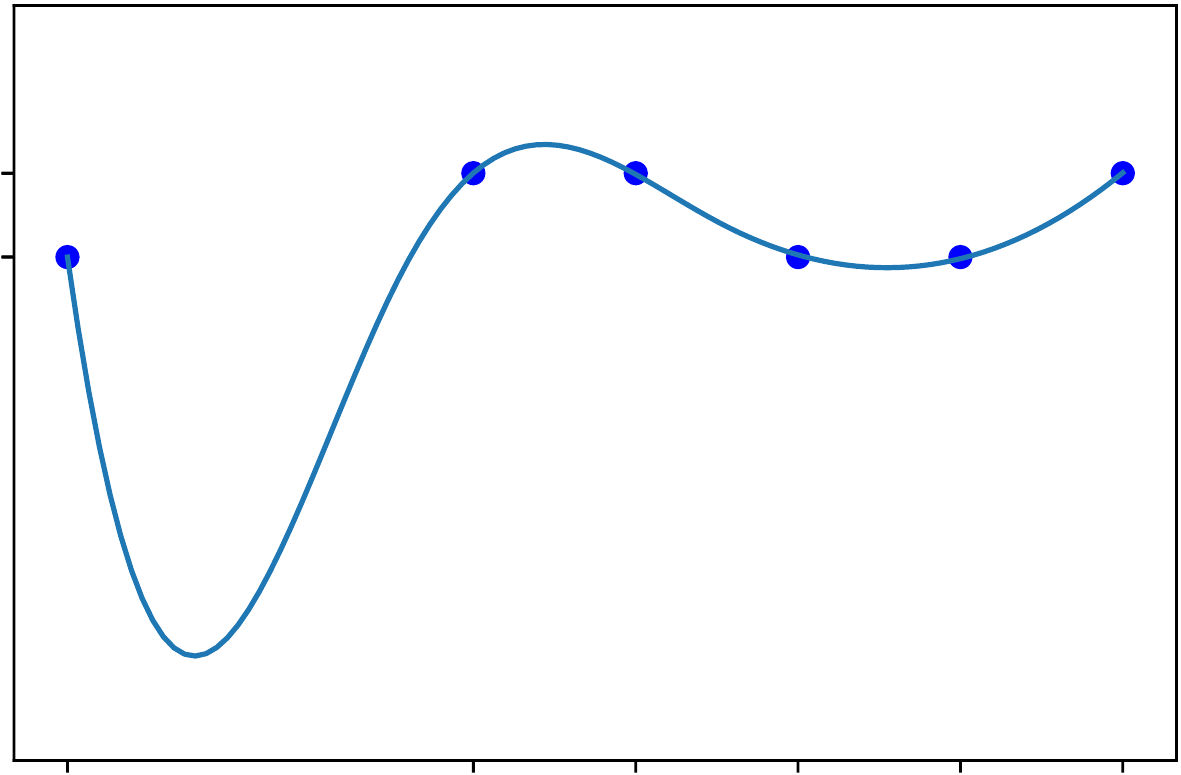}}
  (f) Neural network \\
  without regularization
  $\norm{\D^4 f_\vec{\theta}}_{\M(\R)} = 10.0$
\end{minipage}
\caption{In (a) (resp. (c)) we have the standard linear  (resp. cubic) spline of the data. In (b) (resp. (e)) we have a ReLU (resp. cubic truncated power function) network with $K = 200$ neurons trained with regularization according to \cref{eq:training}. In (c) (resp. (f)) we illustrate the importance of regularization. All figures plot the function (spline or neural network) vs. the input. The dots are the data.}
\vspace{-1.5em}
\label{fig:experiment}
\end{figure}

In \cref{fig:experiment}, we trained a width $K = 200$ ReLU network according to \cref{eq:training} ($\gamma = 2$) and a width $K = 200$ cubic truncated power function network according to \cref{eq:training} ($\gamma = 4$). The choice of $K = 200$ was chosen so that the networks are sufficiently wide according to \cref{thm:main}. We compare the learned functions to the standard linear and cubic splines\footnote{The standard splines were computed using SciPy.}. We also illustrate the importance of regularization by also training the networks without regularization and show that they do not learn the optimal spline interpolations of the data. Indeed, we see in \cref{fig:experiment}(c) that there are extra ``bumps'' between the first and second data point and between the second and third data point, and we see in \cref{fig:experiment}(f) that there is an extra ``bump'' between the first and second data point. While the function learned in \cref{fig:experiment}(b) is not the connect-the-dots linear spline, we see that it has the same second-order total variation and is hence a minimizer to the variational problem.

\section{Conclusion \& Future Work} \label{sec:conclusion}

Using tools from the variational framework of $\Ell$-splines, we have shown that the choice of activation implicitly defines a neural network regularizer that corresponds to a seminorm that defines a Banach space. We showed that the resulting neural network regularizers are related to the well-known path-norm and weight decay regularizers. Finally, we verified our results with empirical validation by showing that trained neural networks are optimal splines fit to data. Understanding the functional characteristics of deep neural networks trained on data is an open question.

\section*{Acknowledgment}
The authors would like to thank Jordan Ellenberg for suggesting the simple argument that appears in \cref{thm:power-necessary}.

\bibliographystyle{IEEEtran}
\bibliography{ref}

\begin{thebibliography}{10}
\providecommand{\url}[1]{#1}
\csname url@samestyle\endcsname
\providecommand{\newblock}{\relax}
\providecommand{\bibinfo}[2]{#2}
\providecommand{\BIBentrySTDinterwordspacing}{\spaceskip=0pt\relax}
\providecommand{\BIBentryALTinterwordstretchfactor}{4}
\providecommand{\BIBentryALTinterwordspacing}{\spaceskip=\fontdimen2\font plus
\BIBentryALTinterwordstretchfactor\fontdimen3\font minus
  \fontdimen4\font\relax}
\providecommand{\BIBforeignlanguage}[2]{{%
\expandafter\ifx\csname l@#1\endcsname\relax
\typeout{** WARNING: IEEEtran.bst: No hyphenation pattern has been}%
\typeout{** loaded for the language `#1'. Using the pattern for}%
\typeout{** the default language instead.}%
\else
\language=\csname l@#1\endcsname
\fi
#2}}
\providecommand{\BIBdecl}{\relax}
\BIBdecl

\bibitem{uat1}
G.~Cybenko, ``Approximation by superpositions of a sigmoidal function,''
  \emph{Mathematics of control, signals and systems}, vol.~2, no.~4, pp.
  303--314, 1989.

\bibitem{uat2}
K.~Hornik, M.~Stinchcombe, and H.~White, ``Multilayer feedforward networks are
  universal approximators,'' \emph{Neural networks}, vol.~2, no.~5, pp.
  359--366, 1989.

\bibitem{uat3}
K.-I. Funahashi, ``On the approximate realization of continuous mappings by
  neural networks,'' \emph{Neural networks}, vol.~2, no.~3, pp. 183--192, 1989.

\bibitem{uat4}
A.~R. Barron, ``Universal approximation bounds for superpositions of a
  sigmoidal function,'' \emph{IEEE Transactions on Information theory},
  vol.~39, no.~3, pp. 930--945, 1993.

\bibitem{nn-approx-non-poly}
M.~Leshno, V.~Y. Lin, A.~Pinkus, and S.~Schocken, ``Multilayer feedforward
  networks with a nonpolynomial activation function can approximate any
  function,'' \emph{Neural networks}, vol.~6, no.~6, pp. 861--867, 1993.

\bibitem{approx-relu-relu-squared}
J.~M. Klusowski and A.~R. Barron, ``Approximation by combinations of {ReLU} and
  squared {ReLU} ridge functions with {$\ell^1$} and {$\ell^0$} controls,''
  \emph{IEEE Transactions on Information Theory}, vol.~64, no.~12, pp.
  7649--7656, 2018.

\bibitem{dimension-independent-approx-bounds}
H.~N. Mhaskar, ``Dimension independent bounds for general shallow networks,''
  \emph{Neural Networks}, vol. 123, pp. 142--152, 2020.

\bibitem{weight-decay}
A.~Krogh and J.~A. Hertz, ``A simple weight decay can improve generalization,''
  in \emph{Advances in neural information processing systems}, 1992, pp.
  950--957.

\bibitem{moody1992effective}
J.~E. Moody, ``The effective number of parameters: An analysis of
  generalization and regularization in nonlinear learning systems,'' in
  \emph{Advances in neural information processing systems}, 1992, pp. 847--854.

\bibitem{generalization-deep-learning}
B.~Neyshabur, S.~Bhojanapalli, D.~Mcallester, and N.~Srebro, ``Exploring
  generalization in deep learning,'' in \emph{Advances in Neural Information
  Processing Systems}, 2017, pp. 5947--5956.

\bibitem{regularization-matters}
C.~Wei, J.~D. Lee, Q.~Liu, and T.~Ma, ``Regularization matters: Generalization
  and optimization of neural nets vs their induced kernel,'' in \emph{Advances
  in Neural Information Processing Systems}, 2019, pp. 9709--9721.

\bibitem{L-splines}
M.~Unser, J.~Fageot, and J.~P. Ward, ``Splines are universal solutions of
  linear inverse problems with generalized {TV} regularization,'' \emph{{SIAM}
  Review}, vol.~59, no.~4, pp. 769--793, 2017.

\bibitem{leaky-relu}
A.~L. Maas, A.~Y. Hannun, and A.~Y. Ng, ``Rectifier nonlinearities improve
  neural network acoustic models,'' in \emph{ICML Workshop on Deep Learning for
  Audio, Speech and Language Processing}, 2013.

\bibitem{path-norm}
B.~Neyshabur, R.~R. Salakhutdinov, and N.~Srebro, ``Path-{SGD}: Path-normalized
  optimization in deep neural networks,'' in \emph{Advances in Neural
  Information Processing Systems}, 2015, pp. 2422--2430.

\bibitem{skip-connections}
K.~He, X.~Zhang, S.~Ren, and J.~Sun, ``Deep residual learning for image
  recognition,'' in \emph{The IEEE Conference on Computer Vision and Pattern
  Recognition (CVPR)}, June 2016.

\bibitem{inv-prob1}
H.~Gupta, J.~Fageot, and M.~Unser, ``Continuous-domain solutions of linear
  inverse problems with {T}ikhonov versus generalized {TV} regularization,''
  \emph{IEEE Transactions on Signal Processing}, vol.~66, no.~17, pp.
  4670--4684, 2018.

\bibitem{inv-prob2}
T.~Debarre, J.~Fageot, H.~Gupta, and M.~Unser, ``B-spline-based exact
  discretization of continuous-domain inverse problems with generalized {TV}
  regularization,'' \emph{IEEE Transactions on Information Theory}, 2019.

\bibitem{relu-sparse}
X.~Glorot, A.~Bordes, and Y.~Bengio, ``Deep sparse rectifier neural networks,''
  in \emph{Proceedings of the fourteenth international conference on artificial
  intelligence and statistics}, 2011, pp. 315--323.

\bibitem{relu-fast}
Y.~LeCun, Y.~Bengio, and G.~Hinton, ``Deep learning,'' \emph{Nature}, vol. 521,
  no. 7553, pp. 436--444, 2015.

\bibitem{relu-linear-spline}
P.~H.~P. Savarese, I.~Evron, D.~Soudry, and N.~Srebro, ``How do infinite width
  bounded norm networks look in function space?'' in \emph{Conference on
  Learning Theory, {COLT} 2019, 25-28 June 2019, Phoenix, AZ, {USA}}, 2019, pp.
  2667--2690.

\bibitem{representer-deep}
M.~Unser, ``A representer theorem for deep neural networks,'' \emph{Journal of
  Machine Learning Research}, vol.~20, no. 110, pp. 1--30, 2019.

\bibitem{aziznejad2019deep}
S.~Aziznejad and M.~Unser, ``Deep spline networks with control of {L}ipschitz
  regularity,'' in \emph{ICASSP 2019-2019 IEEE International Conference on
  Acoustics, Speech and Signal Processing (ICASSP)}.\hskip 1em plus 0.5em minus
  0.4em\relax IEEE, 2019, pp. 3242--3246.

\bibitem{ridge-splines}
R.~Parhi and R.~D. Nowak, ``Neural networks, ridge splines, and {TV}
  regularization in the {R}adon domain,'' \emph{arXiv preprint
  arXiv:2006.05626v1}, 2020.

\bibitem{fractional-splines}
M.~Unser and T.~Blu, ``Fractional splines and wavelets,'' \emph{SIAM review},
  vol.~42, no.~1, pp. 43--67, 2000.

\bibitem{adagrad}
J.~Duchi, E.~Hazan, and Y.~Singer, ``Adaptive subgradient methods for online
  learning and stochastic optimization,'' \emph{Journal of Machine Learning
  Research}, vol.~12, no. Jul, pp. 2121--2159, 2011.

\end{thebibliography}

\end{document}